\documentclass[runningheads]{llncs}
\usepackage[T1]{fontenc}
\usepackage{amsfonts}
\usepackage{amsmath}
\usepackage{amssymb}
\usepackage{mathrsfs}
\usepackage{mathtools}
\usepackage{bbm}
\usepackage{xfrac}
\usepackage{graphicx}
\usepackage[hidelinks]{hyperref}
\usepackage{fca}
\usepackage{color}
\usepackage{orcidlink}
\usepackage{todonotes}
\usepackage{algorithm}
\usepackage{algpseudocode}
\usepackage{cleveref}
\usepackage{acro}
\usepackage{subcaption}
\usepackage{multicol}
\usepackage{microtype}

\let\cref\Cref

\DeclareAcronym{fca}{
  short = FCA,
  long = Formal Concept Analysis
}
\begin{document}

\title{What is the \emph{intrinsic} dimension of your binary data? --- and how to compute it quickly}
\titlerunning{ID of your binary data}

\author{%
  Tom Hanika\inst{1}\orcidlink{0000-0001-7813-9799} \and
  Tobias Hille\inst{2,3}\orcidlink{0000-0002-4918-6374}
}

\authorrunning{T. Hanika \& T. Hille}

\institute{
  Intelligent Information Systems, University of Hildesheim, Hildesheim, Germany\and
  Knowledge \& Data Engineering Group, University of Kassel, Kassel, Germany\and
  Interdisciplinary Research Center for Information System Design\\
  University of Kassel, Kassel, Germany\\
  \email{tom.hanika@uni-hildesheim.de, hille@cs.uni-kassel.de}\\
}
\maketitle              
\begin{abstract}
  Dimensionality is an important aspect for analyzing and understanding
  (high-dimensional) data. In their 2006 ICDM paper Tatti et al.\ answered the
  question for a (interpretable) dimension of binary data tables by introducing
  a normalized correlation dimension. In the present work we revisit their results
  and contrast them with a concept based notion of intrinsic dimension (ID)
  recently introduced for geometric data sets. To do this, we present a novel
  approximation for this ID that is based on computing concepts only
  up to a certain support value. We demonstrate and evaluate our approximation
  using all available datasets from Tatti et al., which have between 469
  and 41271 extrinsic dimensions.
\end{abstract}

\section{Introduction}
\label{sec:introduction}
A deeper understanding of the dimensionality of a (binary) data set helps with many tasks related to data analysis. Possible examples are comparing data sets,  dimension reduction (e.g., by selecting features), or estimating the information actually available or necessary for machine learning procedures. 

The two basic notions of dimensionality are extrinsic and intrinsic dimension. The former can be understood as the number information pieces that we have at hand. For example, the number of columns in a binary data table (BD). However, not all these pieces may contribute equally or at all to the analysis of the data. The intrinsic dimensionality (ID) on the other hand shall reflect the number of required information pieces that are needed to describe each object in the data set. It is important to understand that the ID is therefore directly dependent on the analytical methods used. For example, one may analyze the objects of a BD using (artificially introduced) distance functions~\cite{tatti2007distances,li2006unified}. This would entail a different ID than the analysis using statistical measures~\cite{Chavez} or formal concepts~\cite{kuznetsov2023formal}.

A new perspective on the notion of the intrinsic dimension was recently presented by Hanika et al.~\cite{Hanika2022IntrinsicDO}. The approach chosen there for geometric data sets is universal in that it allows for metric, conceptual, and other features that can be represented by 1-Lipschitz functions. It is based on an axiomatic approach which is based on theoretical work by V. Pestov~\cite{Pestov1,Pestov2}. The main advantage of the new approach to ID is that it actually uses the formal concepts as a measurement tool for determining the intrinsic dimension. This novel ID has already found practical application in the field of FCA, especially for comparing data sets from the same domain~\cite{kuznetsov2023formal}. However, this novelty and the calculation effort required to determine this ID pose a problem.

We therefore want to address the following challenges in this paper. The ID requires to compute all formal concept, as they are the measurement tools. Computing these on high-dimensional and large data sets might pose a problem. We thus investigate theoretically to which extent the ID can be estimated using concepts having some minimum support $s\in[0,1]$. Consecutively, we evaluate experimentally this estimation based on classical item set mining data sets of medium and large dimension. In particular, we use all available data sets from Tatti et al.~\cite{Tatti2006DimensionWhatIT}. To this end, we describe an algorithm that efficiently computes the ID for a given set of concepts and a data set. 
We then revisit the results by Tatti et al. on the (normalized) correlation dimension and contrast them with our ID results. Specifically, we can show that the ID captures other aspects of the data than the correlation dimension.

All in all, with our work we present a computationally feasible approach for the intrinsic dimension of binary datasets that a) does not artificially introduce a distance function on the dataset and b) uses the formal concepts from FCA as a measurement tool. This confirms the previous application results~\cite{kuznetsov2023formal,Hanika2022IntrinsicDO,StubbemannHS23} for this ID and enables its broad use in the future.

\section{Related Work}
There is an abundance of works discussing the intrinsic dimension of data sets. There rather simple approaches, such as Chavez et al.~\cite{Chavez}, that employ comparatively simple statistics.
More sophisticated works link the intrinsic dimension to the mathematical phenomenon of measure concentration~\cite{Pestov1,Pestov2}. From this Hanika et al.\ derived an axiomatization which resulted into a practically computable dimension function~\cite{Hanika2022IntrinsicDO}. Since their result also allows to employ formal concepts for probing data sets, we will build our work on it. 
A slightly different approach is followed by Carter~\cite{DBLP:journals/tsp/CarterRH10}, who proposed to consider a local variant of the intrinsic dimension. Sutton et al.~\cite{DBLP:journals/corr/abs-2311-07579}, on the other hand, propose a relative intrinsic dimension between data distributions.

All works above (with exception for Hanika et al.) have in common that they require the data set in question to have a (meaningful) metric function. Formal context, naturally, do not exhibit such a metric. Although it can artificially introduced, e.g., using Hamming distances, their meaningfulness is debatable. 

For the special case of dimension considerations for binary data
tables, the research literature is very limited. Liu et
al.~\cite{liu2016comparison} propose an algorithmic approach for the
dimensionality assessment of categorical item responses. For this they
employ different notions of multidimensional clustering.

Notwithstanding this, there is some work on dimension reduction in
formal concept analysis. Belohlavek and
Trnecka~\cite{DBLP:journals/jcss/BelohlavekT15} propose different
geometrically motivated Boolean matrix factorization
techniques~\cite{DBLP:journals/isci/BelohlavekOT18,DBLP:journals/kbs/TrneckaT18}. Bartl
et al. provides an in-depth comparison the effects of different
factorization methods on the
dimensionality~\cite{DBLP:conf/chdd/BartlBOR12}.  Buzmakov et
al.~\cite{DBLP:journals/ijar/BuzmakovDKMN24} consider significant fca
notions, such as minimal generators, pseudo-intents, etc, to measure
data complexity.
Finally, and most prominently, is the work by Tatti et al.~\cite{Tatti2006DimensionWhatIT}. Their initial question for ``What is the dimension of your binary data?'' is the starting point for this work on the intrinsic dimension of binary data tables.

\section{Approaches for dimension measurement}
As the related work indicates, there is a plenitude of approaches for
dimensionality. In this work, we provide an in depth comparison of the
\emph{correlation dimension} by Tatti et
al.~\cite{Tatti2006DimensionWhatIT} and the intrinsic dimension of
geometric data sets~\cite{Hanika2022IntrinsicDO}. The former is an
instance of a fractal dimension and uses distances in an Euclidean
space in which the data points of the binary table are
represented. The latter is a non-metric approach which employs the
formal concepts themselves as measuring instruments, as we will show.

\subsection{Data Structures For Binary Data Tables}
There are several ways to build up a data structure for binary data tables.
Tatti et al.~\cite{Tatti2006DimensionWhatIT} used a plain notion, which is an
instance of a matrix $D\in{\{0, 1\}}^{g\times m}$, where $g,m\in\mathbb{N}$. In
this work we represent data of this form by incidence structures, which are
thoroughly investigated in \ac{fca}. Although we will refrain from giving a
detailed introduction and refer the reader to Ganter and
Wille~\cite{Ganter2012FormalCA}, we recall the basic notions of \ac{fca}
relevant to this work. A triple $\mathbb{K}=(G, M, I)$ with non-empty sets $G$
and $M$ and a relation $I\subseteq G\times M$ is called a \textit{formal
  context}. The elements of $G$ are called \textit{objects}, the ones of $M$
\textit{attributes} of $\mathbb{K}$. We call a formal context \textit{finite} if
both $G$ and $M$ are finite. For $A\subseteq G$ and $B\subseteq M$ define
$A':=\{b\in M\mid\forall a\in A:
(a, b)\in I\}$ and $B':=\{a\in G\mid\forall b\in B:(a, b)\in I\}$.
A central notion in \ac{fca} is that of the \textit{formal concepts} of
$\mathbb{K}$, defined as
$\mathcal{B}(\mathbb{K}):=\{(A, B)\mid A\subseteq G, B\subseteq M, A'=B, B'=A\}$.


Denote for $n\in\mathbb{N}$ the set $\{1, \dots, n\}=[n]$. We can naturally
represent a binary data table as a formal context $\mathbb{K}=(G, M, I)$ through
setting $G=[g], M=[m]$ and $I=\{(a, b)\mid \forall a\in G, b\in M: D_{a, b} =
1\}$. For $a\in[g]$ let a \textit{data point} (e.g., a row) of $D$ be
$\mathbf{b}=[D_{a, 1},\dots,D_{a, m}]\in\{0,1\}^m$ and vice versa for $b\in[m]$
let a \textit{feature} (e.g., a column) of $D$ be $\mathbf{a}=[D_{1,
  b},\dots,D_{g,b}]\in\{0,1\}^g$. Then for $A=\{a\}$ a data point $\mathbf{a}$
corresponds to $A'$ (and likewise for $B=\{b\}$ a feature $\mathbf{b}$ to $B'$).
We define the set of data points to be $\mathcal{P}(D)=\{\mathbf{a}\mid
a\in[g]\}$ and call $g$ the \emph{length} of the binary data table.


\subsection{Correlation Dimension}%
\label{sec:approaches:correlation_dimension}
Most definitions of (intrinsic) dimension build upon the idea of measuring a set
in a way that ignores irregularities below a certain scale and observe the
change in behavior when the scale shrinks. This is particularly evident in the
\emph{fractal dimension} and the related \emph{box counting method}
where an scaling exponent is estimated~\cite{Falconer1990FractalGM}. Tatti et
al.~\cite{Tatti2006DimensionWhatIT} have applied these notions, which were
previously defined for real-valued metric spaces, to the special case of binary
data tables.

Given a binary data table $D\in{\{0, 1\}}^{g\times m}$, $Z_D$ denotes the random
variable whose value is the $L_1$ distance between two randomly chosen data
points from $D$. As $0\leq Z_D\leq m$, we can obtain the probability of $Z_D$
being less than a chosen value $r\leq m$ by counting all those pairs that are
close enough, namely:
\begin{align*}
  \mathbb{P}[Z_D< r]
  &= \frac{1}{|D|^2}\big|\{(x, y)\mid x,y\in\mathcal{P}(D): |x-y| < r\}\big|
\end{align*}
The exact calculation is quadratically dependent on the length of $D$. An easy
way to approximate $\mathbb{P}[Z_D<r]$ is by choosing a random subset of rows
$\mathcal{P}_s(D)\subset\mathcal{P}(D)$ and calculate all pairs of elements with
one or both elements in $D_s$:
\begin{align*}
  \mathbb{P}[Z_D< r]
  &\simeq \frac{1}{|D||D_s|}\big|\{(x, y)
    \mid x\in\mathcal{P}(D), y\in\mathcal{P}_s(D): |x-y| < r\}\big| \\
  \mathbb{P}[Z_D< r]
    &\simeq \frac{1}{|D_s|^2}\big|\{(x, y)
      \mid x,y\in\mathcal{P}(D): |x-y| < r\}\big|
\end{align*}
The cumulative probability function of $Z_D$ is a discontinuous step function in
any of the above forms, namely:
$f^\star:\mathbb{N}\rightarrow\mathbb{R}, r\mapsto f(r)=\mathbb{P}[Z_D < r]$

Its extension to real numbers can be obtained~\cite{Tatti2006DimensionWhatIT} by
interpolating nearby support points:
\begin{align*}
  f(r)& =\begin{cases}
         1 & \textit{if } m < r \\
         0 & \textit{if } r < 0 \\
           (r - \lfloor r \rfloor) f^\star(\lfloor r \rfloor)
           + (\lceil r \rceil - r) f^\star(\lceil r \rceil) & \textit{else}
       \end{cases}
\end{align*}
In order to define the correlation dimension, the function $f$ is used to set
set $\mathcal{I}(D, r_1, r_2, N) \coloneqq
  \{(\log r, \log f(r))\mid r = r_1 + \frac{i}{N} (r_2 - r_1), i\in\{0,\dots,N\}\}$,
in which bounds on the distances $r_1, r_2$ and number of sample points
$N\in\mathbb{N}$ determine for which magnitudes the data set is measured.
Similarly to the original work we will omit $N$ in the following for the sake of
brevity.

If for a given $r$ the number of pairs of points within distance $d$
grows as $r^d$ increases, the point set $\mathcal{I}(D, r_1, r_2)$
appears as a straight line with $d$ as its slope. The underlying
assumption $\mathbb{P}[Z_D < r]\propto r^d$ may not be fulfilled for
all data sets, however, Tatti\,et\,al.\ define the \emph{correlation
  dimension} as the factor minimizing the least-square error between
$\log r$ and $\log f(r)$. We formalize this by
\begin{align*}
  \text{cd}_R(D; r_1, r_2) :=
  \operatorname*{arg\,min}_b
  \sum_{(\log r, \log f(r))\in \mathcal{I}(D, r_1, r_2)}
  (b\log r - \log f(r))^2.
\end{align*}
For comparing data sets with different number of features it is helpful to use
the possibility of reformulating the correlation dimension based on thresholds
$\alpha_i\in[0, 1]$, $i\in\{1, 2\}$ by setting $r_i = \max\{f^{-1}(\alpha_i),
1\}$:
\begin{align*}
  \mathcal{I}(D, \alpha_1, \alpha_2)=\mathcal{I}(D, r_1, r_2)\quad\text{and}\quad
  \text{cd}_A(D; \alpha_1, \alpha_2)\coloneqq\text{cd}_R(D; r_1, r_2)
\end{align*}
The authors state that the truncation of the $r_i$ is necessary for algorithmic
stability when working with sparse data sets.

\subsubsection{Normalized Correlation Dimension}
The original authors~\cite{Tatti2006DimensionWhatIT} state that the
above definition does not produce very intuitive results. To help with
comparing different data sets the following normalization was
proposed. Its goal is to determine the smallest number of independent
random Bernoulli variables of a given probability, for which samples
with same marginal probability are indistinguishable from a BD table
$D$ when considering the correlation dimension.

Consider the vector $p=(p_1,\dots,p_m)$ of probabilities where
$p_b=\frac{1}{g}\sum_{a=1}^{g}D_{a,b}$ for each $b\in[m]$. Let us denote by
$\text{ind}(D)$ a random variable constructed by combining $m$ independent
random Bernoulli variables, each with mean $p_b$ and call it the
\textit{independent data set over $D$}. Further consider $h\in\mathbb{N}$ and
$q\in[0, 1]$ and use $h$ independent random Bernoulli variables, all with mean
$q$, to construct another random variable. We will refer to this variable by
$\text{ind}(h, q)$ and call it an \textit{independent data set of width $h$
  (with probability $q$)}.

The two notions of $\text{ind}(D)$ and $\text{ind}(h, q)$ allow to compare the
marginal frequencies of the columns of $D$ without requiring that all variables
have the same marginal frequency. This is because $\text{ind}(D)$ can be
understood as an expectation of $D$ with randomized column ordering. The
comparison can now be done by finding a probability $s\in[0, 1]$ such that both
correlation dimensions are equal, or more formally:
\begin{align*}
  \text{cd}_A(\text{ind}(m, s); \alpha, 1-\alpha) =
  \text{cd}_A(\text{ind}(D); \alpha, 1-\alpha).
\end{align*}
The parameter $\alpha$ is chosen empirically and is set 
the authors to $\frac{1}{4}$.

Finally, given the correlation dimension of $D$ we search for the independent
data set of the smallest size $h$ with probability $s$ that has the same
correlation dimension as $D$.\footnote{For explanations on the monotonicity of
  $\text{cd}_A$ wrt. $\text{ind}(D)$ and other omitted requirements we refer the
  interested reader to the original work~\cite{Tatti2006DimensionWhatIT}.} As
such we formalize the normalized correlation dimension for an acceptable
threshold $\varepsilon\ll 1$ as:
\[
  \text{ncd}_A(D; \alpha, 1 - \alpha, \epsilon)
  = \min_{h\in\mathbb{N}}
  \{h \mid \varepsilon > |\text{cd}_A(\text{ind}(h, s); \alpha, 1 - \alpha) -
  \text{cd}_A(D; \alpha, 1 - \alpha)|\}\]

\subsection{Intrinsic Dimension of Geometric Data Sets}
\label{sec:IDGDS}
We now turn towards the \textit{intrinsic dimension} as introduced by Hanika et
al.~\cite{Hanika2022IntrinsicDO} and briefly recapitulate its mathematical
ingredients.

Consider a set $X$ of \textit{data points}, a set $F\subseteq\mathbb{R}^X$ of
\textit{feature functions} from $X$ to $\mathbb{R}$ and a function of the form
$d_F(x,y)\coloneqq\sup_{f\in F}|f(x)-f(y)|$. It is required that $X$ fulfills
the condition $\sup_{x,y\in X} d_F(x,y)<\infty$. Furthermore let $(X, d_F)$ be a
complete and separable metric space with $\mu$ being a Borel probability measure
on $(X, d_F)$. However, these conditions are met automatically when $X$ not
empty and finite. We call the triple $\mathscr{D} = (X, F, \mu)$ a
\textit{geometric data set}. As it turns out this notion allows for thoroughly
defining a family of dimension functions, of which we want to look into
$\partial(\mathscr{D})$.


As formal concepts are integral to \ac{fca} we use them to define the intrinsic
dimension of a formal context. More specifically we use the size of an extent
(i.e., object set of a formal concept), determined by the normalized counting
measure $\nu_G$, to a given intent (i.e., attribute set of a formal concept) by
constructing the set of feature functions as follows: $F(\mathbb{K}) = \{\nu_G(A)\cdot\mathbbm{1}_B \mid (A, B) \in \mathcal{B}(\mathbb{K})\}$

This gives, for each concept, a map that checks if a given attribute is part of
the intent and then evaluates to the corresponding value of the normalized
counting measure of the extent. In other words, we will measure the data using
the concepts it contains. An advantage over other intrinsic dimension methods is
that this can be done without the need to introduce a metric to the original
data space. Based on the above we define the geometric data set representation
$\mathscr{D}(\mathbb{K})=(M, F(\mathbb{K}), \nu_M)$, where $\nu_M$ is the
normalized counting measure on $M$.

Hanika et al.~\cite{Hanika2022IntrinsicDO} introduce for this representation an
explicit formula to compute the intrinsic dimension.
For each feature function $\nu_G(A)\cdot\mathbbm{1}_B$ a \emph{partial diameter}
is derived by
\begin{align*}
  \text{PartDiam}((\nu_G(A)\cdot\mathbbm{1}_B)_\star(\nu_M), 1-\alpha) =
  \begin{cases}
    \nu_G(A) & \text{if } \alpha < \nu_M(B) < 1 - \alpha \\
    0       & \text{otherwise}
  \end{cases}
\end{align*}
where $(\nu_G(A)\cdot\mathbbm{1}_B)_\star(\nu_M)$ is the push-forward measure of
$\nu_M$. This is then used to compute the \emph{observable diameter}:
\begin{align*}
  \text{ObsDiam}(\mathscr{D}(\mathbb{K}); -\alpha) =
  \sup\{\nu_G(A) \mid (A, B) \in \mathcal{B}(\mathbb{K}),
  \alpha < \nu_M(B) < 1 - \alpha\}
\end{align*}
The resulting value can be interpreted as the fraction of the data that the
features (i.e., the formal concepts) can ``see''. However, only those features
are considered that have an intent size of more than $\alpha$ and less than
$1-\alpha$.

Using this the intrinsic dimension is defined
$\partial_\Delta(\mathscr{D}(\mathbb{K})) =
{(\Delta(\mathscr{D}(\mathbb{K})))}^{-2}$ where
\begin{align*}
  \Delta(\mathscr{D}(\mathbb{K})) =
  \int_0^{\frac{1}{2}}\text{ObsDiam}(\mathscr{D}(\mathbb{K}); -\alpha)\, \mathrm{d}\alpha
\end{align*}

\section{Approximating the ID Geometric Data Sets}

We start by giving a detailed example of the calculations of the geometric
intrinsic dimension for a small formal context. This gives a straightforward
insight on how restrictions on a non trivial minimum support of the considered
formal concepts lead to bounds on the observable diameter and thus the intrinsic
dimension itself. Afterwards we formalize the intuitive notion.

Consider the well know formal context \textit{Living Beings in
  Water}~\cite{Ganter2012FormalCA} with eight objects and nine attributes.
Its corresponding concept lattice has nineteen elements. Let us set
$\mathcal{N}(<0)\coloneqq\{(\nu_G(A),\nu_M(B))\mid (A,B)\in\mathcal{B}(\mathbb{K})\}$ and let $  \mathcal{N}(\alpha)\coloneqq\{(\nu_G(A),\nu_M(B))\mid (A,B)\in\mathcal{B}(\mathbb{K}),
    \alpha<\nu_M(B)<1-\alpha\}$
which results for the example context to
\begin{align*}
  \mathcal{N}(<0) = \{
  &(1, \sfrac{1}{9}),
  (\sfrac{5}{8}, \sfrac{2}{9}),
  (\sfrac{1}{4}, \sfrac{1}{3}),
  (\sfrac{1}{4}, \sfrac{4}{9}),
  (\sfrac{3}{8}, \sfrac{1}{3}), \\
  &(\sfrac{1}{2}, \sfrac{2}{9}),
  (\sfrac{1}{8}, \sfrac{4}{9}),
    (\sfrac{1}{8}, \sfrac{5}{9}),
    (0, 1)
    \},
\end{align*}
i.e., all possible pairs of the sizes of ex- and intents.

We can list all intervals of $\alpha$ for which the observable diameter changes:
\begin{center}
  \begin{tabular}{c@{\hskip 0.5in}c c}
    $\alpha$
    & $\mathcal{N}(\alpha)$
    & $\text{ObsDiam}(\mathscr{D}(\mathbb{K}),-\alpha)$%
    \\
    \hline
    $\alpha < 0$
    & $\mathcal{N}(<0)$
    & $1$ \\
    $0\leq\alpha<\sfrac{1}{9}$
    & $\mathcal{N}(<0)\setminus\{(0, 1)\}$
    & $1$ \\
    $\sfrac{1}{9}\leq\alpha<\sfrac{2}{9}$
    & $\mathcal{N}(0)\setminus\{(1, \sfrac{1}{9})\}$
    & $\sfrac{5}{8}$ \\
    $\sfrac{2}{9}\leq\alpha<\sfrac{3}{9}$
    & $\mathcal{N}(\sfrac{1}{9})\setminus
      \{
      (\sfrac{5}{8},\sfrac{2}{9}),
      (\sfrac{1}{2},\sfrac{2}{9})
      \}$
    & $\sfrac{3}{8}$ \\
    $\sfrac{3}{9}\leq\alpha<\sfrac{4}{9}$
    & 
      $\{
      (\sfrac{1}{4}, \sfrac{4}{9}),
      (\sfrac{1}{8}, \sfrac{4}{9}),
      (\sfrac{1}{8}, \sfrac{5}{9})
      \}$
    & $\sfrac{1}{4}$ \\
    $\sfrac{4}{9}\leq\alpha\leq\sfrac{1}{2}$
    & $\emptyset$
    & $0$ \\
    \hline
  \end{tabular}
\end{center}
For our calculations we set $\sup\emptyset := 0$. The integration of the
observable diameter over values of $\alpha$ between $0$ and $\frac{1}{2}$
simplifies to a sum with finite number of terms, as we are calculating the area
under a stair case function. At this point, compare the following
with the main line in~\Cref{fig:lbiw_gid_bounds}.
\[\Delta(\mathscr{D}(\mathbb{K}))
= \frac{1}{9} +
    \frac{1}{9} \cdot \frac{5}{8} +
    \frac{1}{9} \cdot \frac{3}{8} +
    \frac{1}{9} \cdot \frac{1}{4} = \frac{1}{4} \implies
  \partial_\Delta(\mathscr{D}(\mathbb{K})) = 16\]

\subsection{Minimum Support and $\text{ObsDiam}(\mathscr{D},-\alpha)$}%
\label{sec:approximation:min_support}
Let $s\in[0, 1]$ denote the minimum support and consider the set of all concepts
that have $s$, i.e., $\mathcal{B}_s(\mathbb{K})\coloneqq\{(A,
B)\in\mathcal{B}(\mathbb{K})\mid s\leq\nu_G(A)\}$. By restricting the set of
feature functions $F_s(\mathbb{K})\coloneqq\{\nu_G(A)\cdot \mathbbm{1}_B\mid
(A,B)\in\mathcal{B}_s(\mathbb{K})\}$ accordingly, we can construct the 
geometric data set $\mathscr{D}_s(\mathbb{K})\coloneqq(M,F_s,\nu_M)$. 
Thus, we want to motivate the estimation of the ID by discarding
concepts with low support. For this we define the sets
$  \mathcal{N}(<0, s)
  \coloneqq\{(\nu_G(A),\nu_M(B))\mid(A,B)\in\mathcal{B}_s(\mathbb{K})\}$
and $\mathcal{N}(\alpha, s)
  \coloneqq\{(\nu_G(A),\nu_M(B))\mid(A,B)\in\mathcal{B}_s(\mathbb{K}),
    \alpha<\nu_M(B)<1-\alpha\}$.
Analogously to the steps above, we show the computation for minimum support
$s=\sfrac{1}{2}$ and 
$ \mathcal{N}(<0, \sfrac{1}{2}) =
  \{
  (1, \sfrac{1}{9}),
  (\sfrac{5}{8}, \sfrac{2}{9}),
  (\sfrac{1}{2}, \sfrac{2}{9})
  \}$.
Thus the observable diameters changes at fewer values of $\alpha$ and the sum
changes correspondingly:
\begin{center}
  \begin{tabular}{c@{\hskip 0.2in}c c}
    $\alpha$
    & $\mathcal{N}(\alpha, \sfrac{1}{2})$
    & $\text{ObsDiam}(\mathscr{D}_s(\mathbb{K});-\alpha)$
    \\
    \hline
    $\alpha < \sfrac{1}{9}$
    & $\mathcal{N}(<0, \sfrac{1}{2})$
    & $1$ \\
    $\sfrac{1}{9}\leq\alpha<\sfrac{2}{9}$
    & $\{(\sfrac{5}{8}, \sfrac{2}{9}), (\sfrac{1}{2}, \sfrac{2}{9})\}$
    & $\sfrac{5}{8}$ \\
    $\sfrac{2}{9}\leq\alpha\leq\sfrac{1}{2}$
    & $\emptyset$
    & $0$ \\
    \hline
  \end{tabular}
\end{center}
Thus, $
  \Delta(\mathscr{D}_s(\mathbb{K}))
  =\frac{1}{9} +
    \frac{1}{9} \cdot \frac{5}{8} = \frac{13}{72}\implies\partial_\Delta(\mathscr{D}_s(\mathbb{K})) = \frac{5184}{169}\approx 30.7.$
The key consideration is that for $\sfrac{2}{9}\leq\alpha\leq\sfrac{1}{2}$ all
unobserved $\text{ObsDiam}(\mathscr{D}(\mathbb{K});-\alpha)$ have to be lower
than $\sfrac{5}{8}$. Additionally, it is clear that the sum
$\Delta(\mathscr{D}(\mathbb{K}))$ has to be higher than
$\Delta(\mathscr{D}_s(\mathbb{K}))$ because of the missing terms. We can use
this fact to describe upper and lower bounds on
$\Delta(\mathscr{D}(\mathbb{K}))$ and $\partial_\Delta(\mathscr{D}(\mathbb{K}))$,
which we formalize in the following. For the example the bounds are depicted in~\Cref{fig:lbiw_gid_bounds}. 
\begin{figure}[t]
  \centering
  \includegraphics[width=0.5\textwidth]{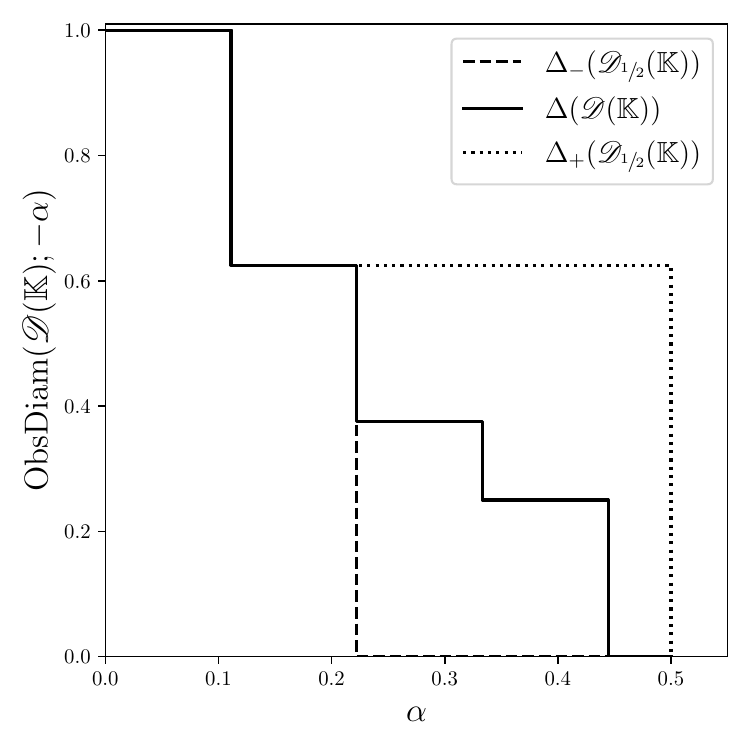}
  \caption{Exemplary visualization of the calculation steps for the geometric intrinsic
    dimension and the corresponding bounds for minimum support
    $\sfrac{1}{2}$. The used symbols for the bounds are introduced at
    the end of~\cref{sec:approximation:min_support}}%
  \label{fig:lbiw_gid_bounds}
\end{figure}
\subsubsection{Bounds on geometric intrinsic dimension}
It follows from $\mathcal{B}_s(\mathbb{K})\subseteq\mathcal{B}(\mathbb{K})$ that
$\text{ObsDiam}(\mathscr{D}(\mathbb{K});-\alpha) \geq
\text{ObsDiam}(\mathscr{D}_s(\mathbb{K});-\alpha)$ holds, which leads to
the inequalities $\Delta(\mathscr{D}(\mathbb{K}))\geq \Delta(\mathscr{D}_s(\mathbb{K}))$ and $  \partial_\Delta(\mathscr{D}(\mathbb{K}))\leq\partial_\Delta(\mathscr{D}_s(\mathbb{K}))$. 
The other bound can also be estimated for finite formal contexts as follows. 
\begin{proposition}
  For $s\in[0,1]$ let 
  $\alpha_{-1}\coloneqq\max\{\alpha\mid\text{ObsDiam}(\mathscr{D}_s(\mathbb{K});-\alpha)>0\}$.
  Then for all $\alpha\in[0,\sfrac{1}{2}]$ the
  $\text{ObsDiam}(\mathscr{D}(\mathbb{K});-\alpha)
  \leq\text{ObsDiam}(\mathscr{D}_s(\mathbb{K});-\alpha_{-1})$.
\end{proposition}
\begin{proof}
  Let $(A_1,B_1)\in\mathcal{B}(\mathbb{K})\setminus\mathcal{B}_s(\mathbb{K})$,
  i.e., a concept with $\nu_G(A_1)<s$. Let  $(A_2, B_2)\in\mathcal{B}_s(\mathbb{K})$ be a concept such that the intent size is $\alpha_{-1}$, i.e., $\nu_{M}(B_{2})=\alpha_{-1}$. From $\nu_G(A_1)<\nu_G(A_2)$ we get $\nu_M(B_1)>\nu_M(B_2)=\alpha_{-1}$. Concepts with support smaller than $s$ can only influence values of ObsDiam with $\alpha$ larger than $\alpha_{-1}$. Thus, for all $\alpha\in[0,\alpha_{-1}]$ the ObsDiam is unchanged. For $\alpha\in(\alpha_{-1},\sfrac{1}{2}]$ the ObsDiam can only become smaller when lowering $s$.
\end{proof}
We continue to use $\alpha_{-1}$ from the proof and let
$\sigma_{-1}\coloneqq\text{ObsDiam}(\mathscr{D}_s(\mathbb{K});-\alpha_{-1})$. This
allows us to define upper and lower bounds
\begin{align*}
  \Delta_-(\mathscr{D}_s(\mathbb{K}))
  &\coloneqq \Delta(\mathscr{D}_s(\mathbb{K})),\quad
    \Delta_+(\mathscr{D}_s(\mathbb{K}))
    \coloneqq \Delta(\mathscr{D}_s(\mathbb{K})) +
    \left(\sfrac{1}{2} - \alpha_{-1}\right)\cdot \sigma_{-1} \\
  \partial_\Delta^-(\mathscr{D}_s(\mathbb{K}))
  &\coloneqq {(\Delta_+(\mathscr{D}_s(\mathbb{K})))}^{-2},\quad
    \partial_\Delta^+(\mathscr{D}_s(\mathbb{K}))
    \coloneqq {(\Delta_-(\mathscr{D}_s(\mathbb{K})))}^{-2}
\end{align*}
that fulfill the inequalities  $\Delta_-(\mathscr{D}_s(\mathbb{K}))
  \leq\Delta(\mathscr{D}(\mathbb{K}))
    \leq\Delta_+(\mathscr{D}_s(\mathbb{K}))$ and inadvertently  $\partial_\Delta^-(\mathscr{D}_s(\mathbb{K}))
  \leq\partial_\Delta(\mathscr{D}(\mathbb{K})) \leq\partial_\Delta^+(\mathscr{D}_s(\mathbb{K}))$.

\subsection{Data sets}%
\label{sec:approximation:data_sets}
The empirical investigations by Tatti et al.\ were done with the following data
sets: \emph{accidents, courses, kosarak, paleo, pos, retail, webview-1,
  webview-2} and \emph{newsgroups}. We were not able to retrieve \emph{courses} and \emph{paleo}.
We presume that \emph{paleo} was a snapshot~\footnote{\url{https://nowdatabase.org}}, and the creation and transformation of the snapshot is not documented. The
\emph{newsgroups} data set is used by Tatti et al.\ in other experiments, thus we did not include it in our investigation. Additionally, the original data sets from the \emph{KDD-CUP 2000} challenge are
not available anymore, but there exists a partially restored
repository\footnote{\url{https://www.kdd.org/kdd-cup/view/kdd-cup-2000}}.
Unfortunately, it only contains the \emph{pos} and \emph{webview-2} data sets,
but not \emph{webview-1}. All other data sets could be retrieved from a different
website\footnote{\url{http://fimi.uantwerpen.be/data}} and are also included in \emph{scikit-mine}\cite{Inra2022scikitmineAP}. In total we used \emph{accidents,
  kosarak, pos, retail} and \emph{webview-2}. Furthermore we included\footnote{all available
  through \emph{scikit-mine} too} \emph{chess, connect, mushroom, pumsb} and
\emph{pumsb\_star}. \Cref{tab:data_set_info} shows
quantitative descriptions of all data sets and indicate their status regarding
availability.

\begin{table}[t]
  \centering
  \caption{Qualitative descriptions of data sets and their availability.}%
  \label{tab:data_set_info}
  \begin{tabular}{l r r r r@{\hskip 0.1in}c}
    Data Set
    & $|G|$
    & $|M|$
    & $|I|$
    & density
    & available
    \\
    \hline
    Accidents
    & 340183
    &   469
    & 11500870
    & 7.21
    & yes
    \\
    Courses
    & 2405
    &  5021
    &    64743
    & 0.54
    & no
    \\
    Kosarak
    & 990002
    & 41271
    &  8019015
    & 0.02
    & yes
    \\
    Paleo
    &    501
    &   139
    &     3537
    & 5.08
    & no
    \\
    POS
    & 515597
    &  1657
    &  3367020
    & 0.39
    & yes
    \\
    Retail
    &  88162
    & 16470
    &   908576
    & 0.06
    & yes
    \\
    WebView-1
    &  59602
    &   497
    &   149639
    & 0.51
    & no
    \\
    WebView-2
    &  77512
    &  3340
    &   358278
    & 0.14
    & yes
    \\
    \hline
    Chess
    & 3196
    & 75
    & 118252
    & 0.49
    & yes
    \\
    Connect
    & 67557
    & 129
    & 2904951
    & 0.33
    & yes
    \\
    Mushroom
    & 8124
    & 117
    & 178728
    & 0.19
    & yes
    \\
    Pumsb
    & 49046
    & 2113
    & 3629404
    & 0.04
    & yes
    \\
    Pumsb\_star
    & 49046
    & 2088
    & 2475947
    & 0.02
    & yes
    \\
    \hline
  \end{tabular}
\end{table}
\subsection{Algorithm}%
\label{sec:approximation:algorithm}
Given the theoretical considerations in~\Cref{sec:approximation:min_support} the resulting algorithm is straightforward.  The starting point is the enumeration of all formal concepts and
their ex- and intent sizes. This can be obtained through any available concept
miner. However, we note several minor problems in~\Cref{sec:approximation:problems}. For finite contexts, there is only a finite number of values for $\alpha$ at which the ObsDiam changes. 
We thus can employ a stair case pattern, similar to~\cref{fig:lbiw_gid_bounds}, i.e., the set of all possible pairs of ex- and intent sizes $\mathcal{N}(<0, s)$. To get all different values for $\alpha$ at which ObsDiam potentially changes, consider the set
$\mathcal{C}\coloneqq\{\nu_M(B)\mid (\nu_G(A),\nu_M(B))\in\mathcal{N}(<0, s)\}$ and derive
from it a tuple of the form $(\alpha_1,\dots,\alpha_k)$ with $k=|\mathcal{C}|$, such
that the elements of the tuple are increasingly ordered. By defining
$\sigma_i\coloneqq\max\{\nu_G(A)\mid
(\nu_G(A),\nu_M(B))\in\mathcal{N}(<0,s),\nu_M(B)=\alpha_i\}$ for $i\in[k]$ we
can carry out the last preprocessing step by building a tuple (ordered list) from these
values, namely $\mathcal{T}\coloneqq((\alpha_1,\sigma_1),\dots,(\alpha_k,\sigma_k))$.

Consecutively we calculate straightforward the actual values for  $\alpha$ at which ObsDiam changes. The main idea is to start in the middle of $\mathcal{T}$ and move two pointers outward. We then look for changes of the current maximum of $\nu_{G}$ and keep the related $\alpha_{i}$. This results in
$\mathcal{A}\coloneqq((\alpha_1,\sigma_1),\dots,(\alpha_k,\sigma_k))$.\footnote{For
  completeness, we add $(0, 1)$ as the first and $(\sfrac{1}{2}, 0)$ as the last
  point in cases where they do not already exist.} During the procedure we can
collect the individual parts of the integral sum $\Delta$. The whole procedure takes negligible runtime compared to the calculation of concepts. We include the pseudocode in more detail in appendix. \Cref{alg:gid_init} computes the initialization of the necessary pointers and~\Cref{alg:gid} the main routine.

\subsection{Problems}%
\label{sec:approximation:problems}
We used existing tools to calculate the concepts. We experimented with
the following concept miners \emph{fcbo}, \emph{pcbo},
\emph{iteress}~\cite{Outrata2012FastAF,Krajca2010AdvancesIA,%
  Krajca2010ParallelAF,Krajca2008ParallelRA}%
\footnote{all three available at \url{https://fcalgs.sourceforge.net/}}
and \emph{inclose}~\cite{Andrews2009InCloseAF}
\footnote{\url{https://sourceforge.net/projects/inclose/}}.

The tool \emph{inclose} was not able to compute the concepts for
non-trivial sized data sets, as we encountered \textit{stack smashing}
errors repeatedly. Likewise, \emph{iteress} turned out to be unusable
for our experiments, as it outputs only a sub set of all formal
concepts. We hypothesize that these are the factors of the set
representation of the formal
context~\cite[Chapter~11.1]{Ganter2013DiskreteMG}. Then such a result
would need further computation to obtain the remaining concepts.

Therefore, we decided to conduct our experiments with \emph{fcbo} and
\emph{pcbo}. Both algorithms compute only the intent for each
concept. Thus the extent to a given intent has to be calculated
afterwards. We used a simple counting approach over binary arrays to
achieve this. However, in cases of very large data sets combined with
a very large number of intents (i.e., Gigabytes), our direct approach
was infeasible. This was the case for
$\text{accidents}_{0.1},\text{chess}_{0.1,0.2},\text{connect}_{0.0},
\text{pumsb}_{0.4,0.5}$ and $\text{retail}_{0.0}$, where the indices
denote the different values for the minimum support. Additionally,
while using \emph{pcbo} on very large data sets, such as
\emph{kosarak}, we observed overflows of the variable that counts the
closure steps and in some cases \textit{segfault} errors. We conclude
from these observations and problems that for future applications of
FCA on large and high-dimensional data sets a renewal of the classical
concept mining tools is desired and necessary.

\subsection{Results}%
\label{sec:approximation:results}
We applied our method on all available data sets
from Tatti et al.~\cite{Tatti2006DimensionWhatIT} and additionally five more as
described in~\Cref{sec:approximation:data_sets}. All results, and the source code, are documented in detail and can be reproduced~\cite{anonymous_2024_10908237}. As stated above, the runtime for computing the concepts and the related  ex- and intent measures where the limiting factor
regarding the feasibility. We tried using minimum support values
$\{0.0,\sfrac{1}{10},\dots\sfrac{9}{10}\}$. We applied the cut-off time of two days for the computation of the concepts per data set and minimum support value $s$. In~\Cref{tab:gid_results} we
depict the results for the smallest achieved $s$. We included the number of (staircase) steps in $|\mathcal{A}|$ as well as the lower and upper bound on the geometric ID.

\begin{table}[t]
  \centering
  \caption{Approximations of geometric intrinsic dimension.}%
  \label{tab:gid_results}
  \begin{tabular}{r@{\hskip 0.1in}r@{\hskip 0.1in}r@{\hskip 0.1in}r@{\hskip 0.1in}r@{\hskip 0.1in}r}
    Data Set
    & $s$
    & $|\mathcal{B}_s(\mathbb{K})|$
    & $|\mathcal{A}|$
    & $\partial_\Delta^-(\mathscr{D}_s)$
    & $\partial_\Delta^+(\mathscr{D}_s)$
    \\
    \hline
    Accidents
    & $0.2$
    & $887441$
    & $16$
    & $59.6$
    & $2494.7$
    \\
    Kosarak
    & $0.1$
    & $10$
    & $5$
    & $189.8$
    & $1460594324.9$
    \\
    POS
    & $0.1$
    & $14$
    & $4$
    & $117.4$
    & $4481681.0$
    \\
    Retail
    & $0.1$
    & $10$
    & $4$
    & $36.6$
    & $330947622.3$
    \\
    WebView-2
    & $0.0$
    & $1691051$
    & $99$
    & $644467178.1$
    & $887751112.0$
    \\
    Chess
    & $0.2$
    & $22918586$
    & $25$
    & $16.3$
    & $23.1$
    \\
    Connect
    & $0.1$
    & $8037778$
    & $31$
    & $28.6$
    & $39.4$
    \\
    Mushroom
    & $0.0$
    & $227700$
    & $24$
    & $239.5$
    & $239.8$
    \\
    Pumsb
    & $0.6$
    & $1075015$
    & $24$
    & $5.2$
    & $9493.0$
    \\
    Pumsb\_star
    & $0.1$
    & $1513782$
    & $40$
    & $135.1$
    & $4948.7$
    \\
    \hline
  \end{tabular}
\end{table}
Additionally, we depict the calculated (step) functions of
$\text{ObsDiam}$ for the \emph{chess} data set
in~\Cref{fig:chess_alphas}. For each $s$ the curve continues until
$\alpha_{-1}$, where it then drops to zero. The area (under the curve)
thus covered is precisely $\Delta_-(\mathscr{D}_s(\mathbb{K}))$. If we
would continue the horizontal line that runs through
$(\alpha_{-1},\sigma_{-1})$ until $(\sfrac{1}{2},\sigma_{-1})$, then
we obtain exactly the additional area that is described by
$\Delta_+(\mathscr{D}_s(\mathbb{K}))$, compare
to~\Cref{fig:lbiw_gid_bounds}.  Conversely, \Cref{fig:chess_gid}
depicts how the discussed upper and lower bounds of the ID change
depending on the used minimum support value.  Analogous depictions for
the remaining data sets can found in the appendix.
\begin{figure}[t]
  \centering
  \begin{subfigure}[t]{0.5\textwidth}
    \includegraphics[width=\textwidth]{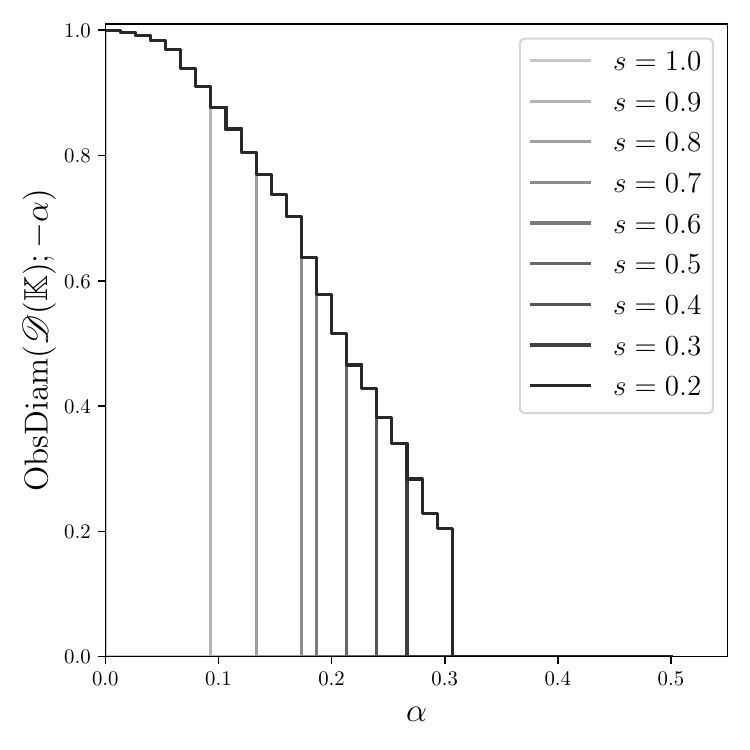}
    \captionsetup{width=.95\linewidth}
    \caption{ObsDiam as a function of $\alpha$ for different minimum support
      each staircase pattern in varying shades of gray, some overlapping.}%
    \label{fig:chess_alphas}
  \end{subfigure}%
~%
  \begin{subfigure}[t]{0.5\textwidth}
    \includegraphics[width=\textwidth]{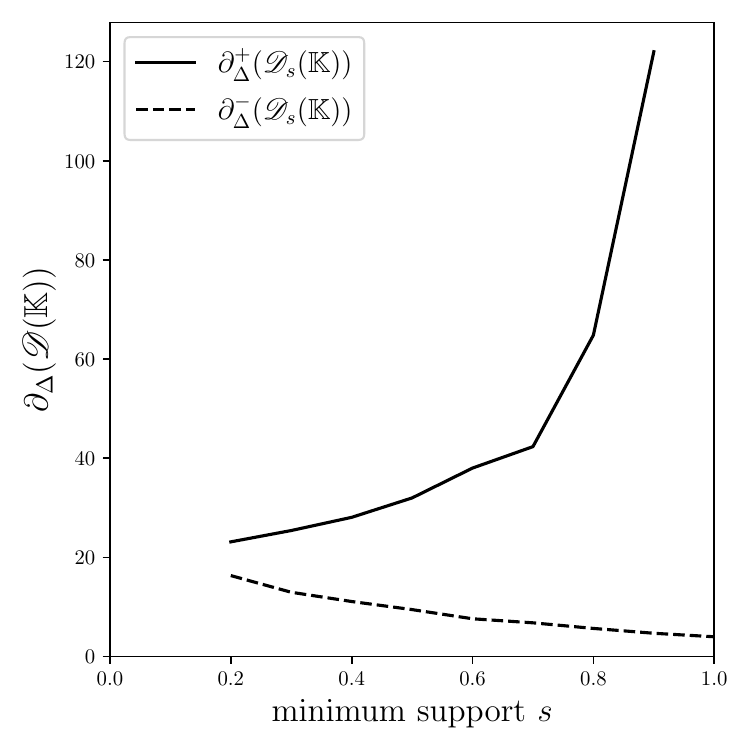}
    \captionsetup{width=.95\linewidth}
    \caption{Lower and upper bound of the geometric intrinsic dimension for
      different minimum support. Missing values are due to infinity (s=1.0) and
      excessive long running times (s<0.2).}%
    \label{fig:chess_gid}
  \end{subfigure}
  \caption{Results of the computation for the \emph{chess} data set. }%
  \label{fig:chess_result}
\end{figure}
\subsection{Observations}
We were able to calculate the concepts with low minimum support for most data
sets.\footnote{A notable exception to that is the \emph{pumsb} data set, where
  the aforementioned necessary determination of the ex- and intent sizes was the
  limiting factor, not the calculation of the concepts itself.}  As some data
sets are large and sparse, it is clear that many concepts may have even lower
support. For such data sets in particular it is apparent that the interval
described by upper and lower bound for the ID is very wide. Thus, we gain
less information even at very low support, especially for the data sets
\emph{kosarak, pos} and \emph{retail}. In the other cases we get estimations for
the ID that is more close, often within same or the next magnitude. 

Additionally, it is interesting to see that the values of ObsDiam
change very few times, as the set $\mathcal{A}$ is very small when
compared to the size of the set concepts. In absolute terms,
\emph{webview-2} appears to be an outlier, as it has a lot of concepts
with different intent sizes, that change the ObsDiam through their
extents as well.\footnote{Note that for support $0.0$ we have
  $\partial_\Delta(\mathscr{D})=\partial_\Delta^-(\mathscr{D}_s)$.}
However, those differences might change considerably if results can be
obtained for the remaining, smaller support values. The new approach
to calculating the ID using the support appears to be feasible and
meaningful. This is especially true if one were to examine more
fine-grained support values.
\section{Dimension Comparison}
To contrast our observations, we compare the results with those presented by
Tatti et al. They calculated $\text{cd}_A(D; \alpha_1, \alpha_2, N)$ with $\alpha_1=\sfrac{1}{4}, \alpha_2=\sfrac{3}{4}, N=50$ and the corresponding
$\text{ncd}_A$, as summarized in~\Cref{sec:approaches:correlation_dimension}.
Additionally, the calculated  value of $\sfrac{\mu}{\sigma}$ was provided, where
$\mu=\mathbb{E}[Z_{\text{ind}(D)}]$ and
$\sigma=\mathbb{V}\!ar[Z_{\text{ind}(D)}]$ are the expectation and variance of
the independent data set over $\mathcal{D}$. As we did not reimplement their
approach, the reported values are taken from Tatti et al. \Cref{tab:dim_comparison}
collects their and our results for the available data sets.

A striking observation is that for four out of five data sets the
normalized correlation dimension is between the bounds on the
geometric intrinsic dimension. However, due to their complete
different nature we might consider this as coincidental. More
important is the observation that both bounds for the intrinsic ID
capture other (dimensional) aspects of the data sets than the
normalized correlation dimension. It is also important that neither
data density nor other properties like like $|M|$, $|G|$, $|I|$,
indicate the behavior of the intrinsic dimension (or more specifically
of its bounds). It is important to note that even for the the fully
calculated ID for the \emph{webview-2} data set there does not seem to
be a direct connection to any other property.
As we explained in~\cref{sec:IDGDS}, the formal concepts are a possible measuring instrument for analyzing binary data tables. Theoretically, a variety of other metric and non-metric feature functions can be considered. However, the great advantage of formal concepts is that they also represent the knowledge units of the FCA. Thus, a different choice of feature functions might measure higher or lower dimension values, as the ncd does, but the benefit for FCA-based analyses might be less valuable.

\begin{table}[t]
  \centering
  \caption{Comparison of correlation dimension and intrinsic dimension.}%
  \label{tab:dim_comparison}
  \begin{tabular}{r r r@{\hskip 0.1in}r@{\hskip 0.1in}r@{\hskip 0.1in}%
    r@{\hskip 0.1in}r@{\hskip 0.1in}r@{\hskip 0.1in}r}
    Data Set
    & $|M|$
    & density
    & $\sfrac{\mu}{\sigma}$
    & $\text{cd}_A$
    & $\text{ncd}_A$
    & $s$
    & $\partial_\Delta^-(\mathscr{D}_s)$
    & $\partial_\Delta^+(\mathscr{D}_s)$
    \\
    \hline
    Accidents
    & $469$
    & $7.21$
    & $6.67$
    & $3.79$
    & $220$
    & $0.2$
    & $59.6$
    & $2494.7$
    \\
    Kosarak
    & $41271$
    & $0.02$
    & $3.96$
    & $0.96$
    & $2378$
    & $0.1$
    & $189$
    & $1460594324.9$
    \\
    POS
    & $1657$
    & $0.39$
    & $3.62$
    & $1.14$
    & $181$
    & $0.1$
    & $117.4$
    & $4481681.0$
    \\
    Retail
    & $16470$
    & $0.06$
    & $4.49$
    & $1.33$
    & $1791$
    & $0.1$
    & $36.6$
    & $330947622.3$
    \\
    WebView-2
    & $3340$
    & $0.14$
    & $3.05$
    & $1.01$
    & $359$
    & $0.0$
    & $644467178.1$
    & $887751112.0$
    \\
    \hline
  \end{tabular}
\end{table}
\section{Conclusion and Outlook}
In conclusion, our work has contributed to the understanding of the ID
of binary data tables. Building on the idea to use formal concepts as
measuring instruments, we have seen that the restriction to concepts
with support larger than a specified threshold leads to
straightforward way to determine lower and upper bounds of the
intrinsic dimension.  In some cases a small minimum support is enough
to give relatively close bounds while reducing the computational
runtime to a feasible amount. However, when a large amount of concepts
is concentrated around small minimum support values, the computed
bounds might be magnitudes apart.

Nevertheless, on the basis of the promising initial results we look to continue this line of research and envision several promising avenues for future investigations. Firstly, the limitation to compute all (supported) formal concepts might be overcome by modifying support-driven algorithms, such as  \emph{Titanic}~\cite{Stumme2002ComputingIC}, towards only computing steps for the ID. This could again  significantly enhance the computational feasibility of obtaining the bounds of the geometric intrinsic
dimension. Secondly, it is advisable to extend the experimental study on the ID of binary data towards a larger and more divers set of data sets.

Lastly, the complex interplay between the set of feature functions and
computational feasibility presents an opportunity for reformulation of which
features one might use. Formal concepts by themselves represent a global view on
the connection between data points. By focusing more on similarities and
differences between individual data points, it may be possible to refine the
feature functions to better capture local aspects of the data.

\bibliographystyle{splncs04}
\bibliography{bibliography}
%
\appendix
\section{Algorithm}
Pseudocode for the algorithm described in~\Cref{sec:approximation:algorithm}.
\begin{algorithm}{Preprocessing}
  \caption{}%
  \label{alg:gid_t}
  \begin{algorithmic}[1]
    \Procedure{Preprocessing}{$\nu_G(A),\nu_M(B),\mathcal{B}(\mathbb{K})$}
    \State $S\gets \{\;\}$%
    \Comment{A map from intent measure to maximum extent measure}
    \For{$(A, B)\in\mathcal{B}(\mathbb{K})$}
    \State $S[\nu_M(B)]\gets\max\{\nu_G(A), S[\nu_M(B)]\}$%
    \Comment{Find largest extent for given intent}
    \EndFor
    \State $T\gets [(a, b) \text{ for } b, a \text{ in
      sorted}(S.items())]$%
    \Comment{Sort map based on intent measure}
    \State \textbf{return} $T$
    \EndProcedure{Preprocessing}
  \end{algorithmic}
\end{algorithm}
\begin{algorithm}[H]
  \renewcommand{\footnotesize}{\fontsize{6pt}{11pt}\selectfont}
  \scriptsize
  \caption{Initialization}%
  \label{alg:gid_init}
  \begin{algorithmic}[1]
    \Procedure{GID\_init}{$T$}
    \If{$\exists i: {\mathcal{T}[i]}_b\leq 0.5$}%
    \Comment{Check if at least one $\alpha$ is below $0.5$}
    \State$I\gets \max\{i \mid \forall j\leq i: {\mathcal{T}[j]}_b\leq 0.5\}$%
    \Comment{Find position of largest $\alpha$ below $0.5$}
    \If{${\mathcal{T}[I]}_b = 0.5$}\Comment{There is a value for exactly $\alpha=0.5$}
    \State$I_l\gets I$, $s_-\gets a_I$\Comment{Set left index pointer and last supremum}
    \Else%
    \State$I_l\gets I - 1$, $s_-\gets 0$\Comment{Set left index pointer and last supremum}
    \EndIf%
    \State$I_u\gets I$\Comment{Set upper index pointer}
    \Else%
    \State$I_l\gets|\mathcal{T}|-1$, $I_u\gets|\mathcal{T}|$, $s_-\gets 0$%
    \Comment{Set pointers and last supremum}
    \EndIf%
    \State\textbf{return} $I_l,I_u,s_-$
    \EndProcedure%
  \end{algorithmic}
\end{algorithm}
\vspace*{-0.5cm}
\begin{algorithm}[H]
  \renewcommand{\footnotesize}{\fontsize{6pt}{11pt}\selectfont}
  \scriptsize
  \caption{Geometric Intrinsic Dimension of Concept Lattice}%
  \label{alg:gid}
  \begin{algorithmic}[1]
    \Procedure{GID}{$\mathcal{T}$}\Comment{see text}
    \State$I_l,I_u,s_-\gets\text{GID\_init}(\mathcal{T})$
    \State$\alpha_-\gets 0.5$\Comment{Store last alpha}
    \State$\mathcal{A}\gets [\;]$\Comment{A list for all realized $\alpha$ thresholds}
    \State$s\gets 0$, $\Delta\gets 0$\Comment{Store intermediate supremum and $\Delta$ sum}
    \While{True}\Comment{Run until both index pointers reach respective end positions}
      \If{$-1 < I_l$ and $I_u < |\mathcal{T}|$}\Comment{Both pointers on the move}
        \State$d^-\gets\alpha_- - {\mathcal{T}[I_l]}_b$, %
        $d^+\gets\alpha_- - (1 - {\mathcal{T}[I_u]}_b)$%
        \Comment{Distance to next threshold}
        \If{$d^- < d^+$}\Comment{Lower threshold is nearer}
          \State$s\gets\max\{s, {\mathcal{T}[I_l]}_a\}$,
          $\alpha\gets{\mathcal{T}[I_l]}_b$, %
          $I_l\gets I_l - 1$\Comment{Move down}
        \ElsIf{$d^- > d^+$}\Comment{Upper threshold is nearer}
          \State$s\gets\max\{s, {\mathcal{T}[I_u]}_a\}$,
          $\alpha\gets{\mathcal{T}[I_u]}_b$, %
          $I_u\gets I_u + 1$\Comment{Move up}
        \Else\Comment{Both distances are the same; Move up and down}
          \State$s\gets\max\{s, {\mathcal{T}[I_l]}_a\}$,
          $\alpha\gets{\mathcal{T}[I_l]}_b$, %
          $I_l\gets I_l - 1$, $I_u\gets I_u + 1$
        \EndIf%
      \ElsIf{$-1 < I_l$ and $I_u = |T|$}\Comment{Lower pointer still moving}
        \State$s\gets\max\{s, {\mathcal{T}[I_l]}_a\}$,
        $\alpha\gets{\mathcal{T}[I_l]}_b$, %
        $I_l\gets I_l - 1$
      \ElsIf{$-1 = I_l$ and $I_u < |T|$}\Comment{Upper pointer still moving}
        \State$s\gets\max\{s, {\mathcal{T}[I_u]}_a\}$,
        $\alpha\gets 1 - {\mathcal{T}[I_u]}_b$, %
        $I_u\gets I_u + 1$
      \Else\Comment{Both pointers reached respective end positions}
        \State\textbf{break}
      \EndIf%
      \State$\mathcal{A}\gets\mathcal{A}+[{[\alpha, s]}]$%
      \Comment{Store realized threshold}
      \State$\Delta\gets\Delta + s_- \cdot (\alpha_- - \alpha)$%
      \Comment{Add next summand}
      \State$\alpha_-\gets\alpha$, $s_-\gets s$\Comment{Update last values}
    \EndWhile%
    \State$\Delta\gets\Delta + \alpha_-\cdot s_-$\Comment{Add last summand}
    \If{$|\mathcal{A}|=0$}\Comment{No realized thresholds}
    \State$\Delta_+\gets\Delta + 0.5$\Comment{Everything is possible}
    \Else%
    \State$\Delta_+\gets{\mathcal{T}[0]}_b\cdot(0.5-{\mathcal{T}[0]}_a)$%
    \Comment{Only rightmost block until first $\alpha$ is possible}
    \EndIf%
    \State$\partial_\Delta(\mathscr{D})\gets%
    \begin{cases}
      \inf & \text{if } \Delta = 0 \\
      \Delta^{-2} & \text{otherwise}
    \end{cases}$
    \State$\partial_\Delta^-(\mathscr{D})\gets\Delta_+^{-2}$%
    \Comment{$\Delta_+$ can never be zero}
    \State\textbf{return} $\partial_\Delta(\mathscr{D}),
    \partial_\Delta^-(\mathscr{D}), \mathcal{A}$\Comment{GID, lower bound
    and realized $\alpha$ thresholds}
    \EndProcedure%
  \end{algorithmic}
\end{algorithm}
\newpage
\section{Additional  plots}
We include additional plots for our empirical investigation. For every
data set we included a graph for the $\text{ObsDiam}$ function as it
changes with varying minimum support and the resulting upper and lower
bounds on the intrinsic dimension. The data sets \emph{kosarak,
  retail, pos} and \emph{webview-2} have very few concepts for most
minimum support values, thus their graphs look more sparse than the
others.
Given their size, the graphs are best read digitally.

\begin{multicols}{2}
  \begin{figure}[H]
    \begin{subfigure}[t]{0.25\textwidth}
      \scriptsize
      \includegraphics[width=\textwidth]{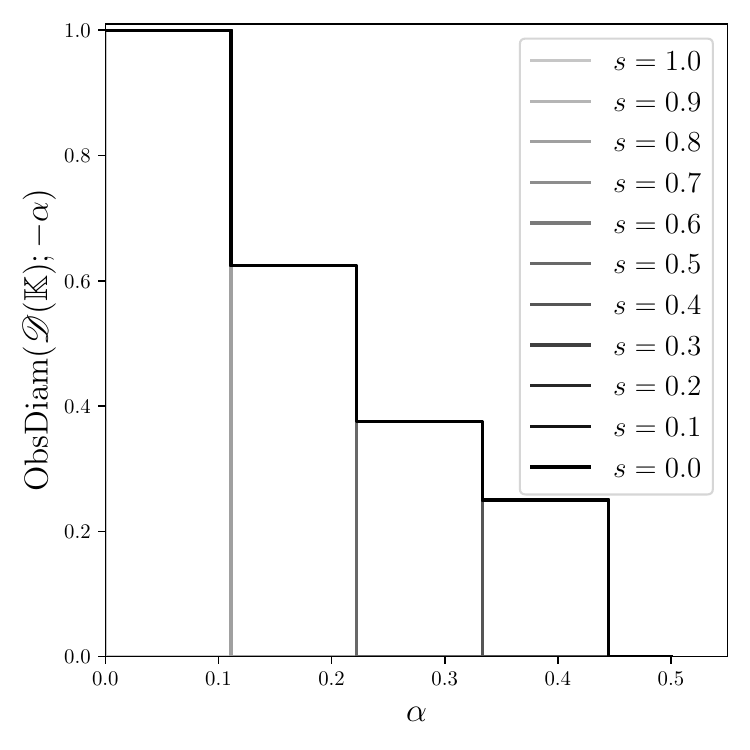}
      \captionsetup{width=.95\linewidth}
      \caption{Same as~\Cref{fig:chess_alphas}}%
      \label{fig:lbiw_alphas}
    \end{subfigure}%
    \begin{subfigure}[t]{0.25\textwidth}
      \scriptsize
      \includegraphics[width=\textwidth]{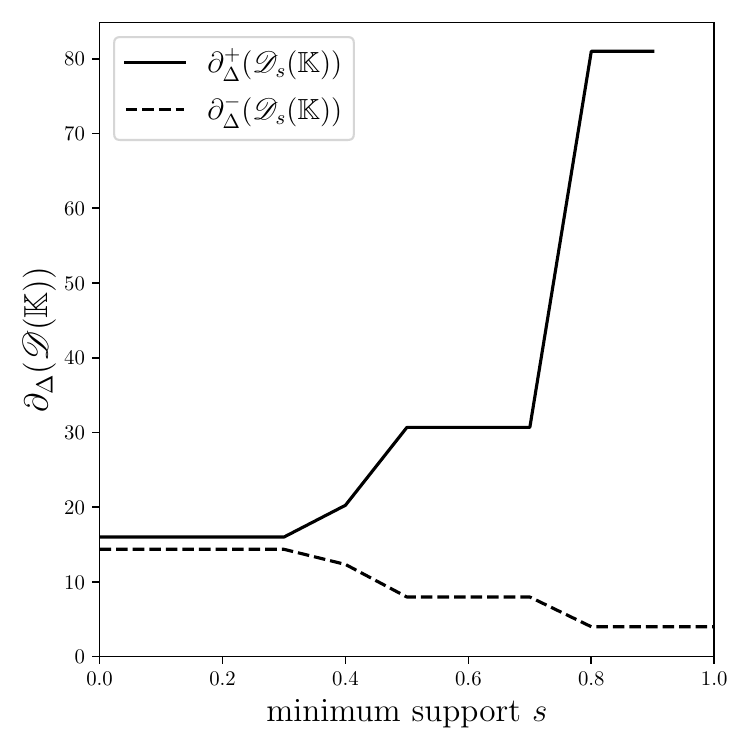}
      \captionsetup{width=.95\linewidth}
      \caption{Same as~\Cref{fig:chess_gid}}%
      \label{fig:lbiw_gid}
    \end{subfigure}
    \caption{Results of the computation for the \emph{living beings in water}
      data set.}%
    \label{fig:lbiw_result}
  \end{figure}
  \columnbreak%
  \begin{figure}[H]
    \begin{subfigure}[t]{0.25\textwidth}
      \scriptsize
      \includegraphics[width=\textwidth]{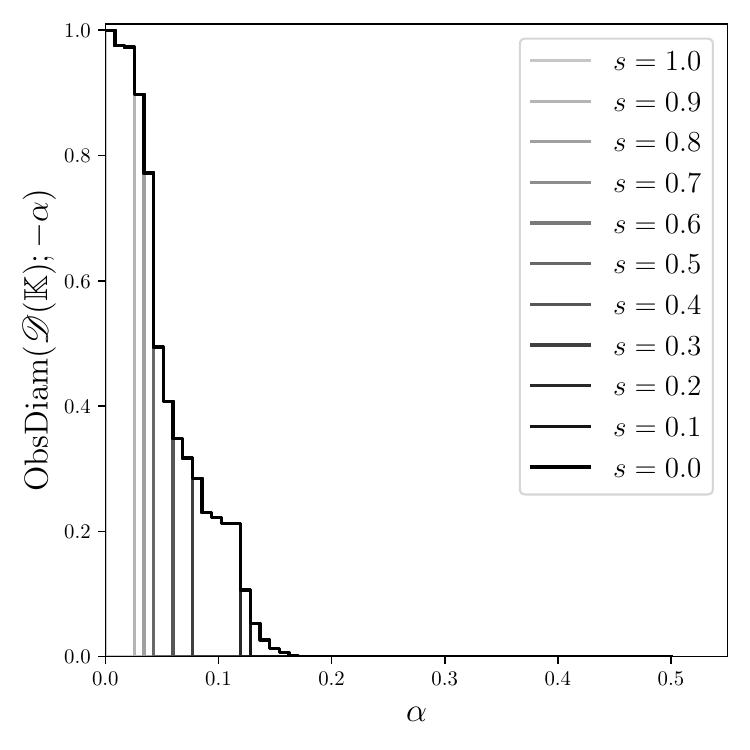}
      \captionsetup{width=.95\linewidth}
      \caption{Same as~\Cref{fig:chess_alphas}}%
      \label{fig:mushroom_alphas}
    \end{subfigure}%
    \begin{subfigure}[t]{0.25\textwidth}
      \scriptsize
      \includegraphics[width=\textwidth]{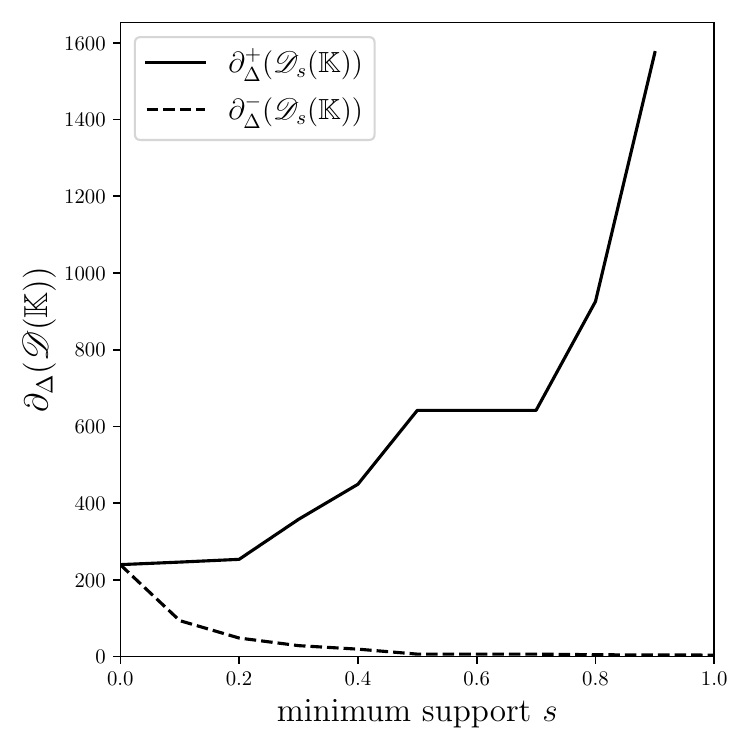}
      \captionsetup{width=.95\linewidth}
      \caption{Same as~\Cref{fig:chess_gid}}%
      \label{fig:mushroom_gid}
    \end{subfigure}
    \caption{Results of the computation for the \emph{mushroom} data set.}%
    \label{fig:mushroom_result}
  \end{figure}
\end{multicols}
\begin{multicols}{2}
  \begin{figure}[H]
    \begin{subfigure}[t]{0.25\textwidth}
      \scriptsize
      \includegraphics[width=\textwidth]{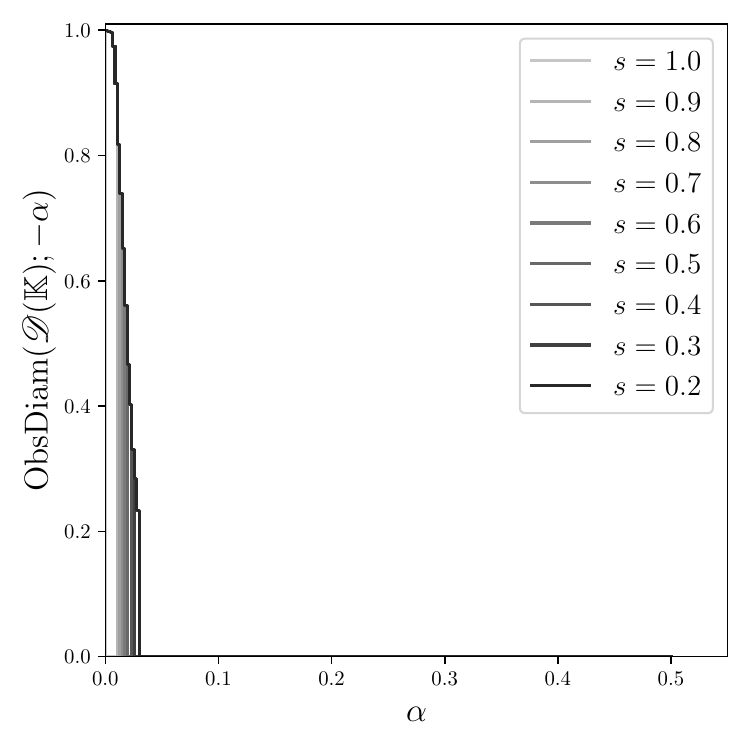}
      \captionsetup{width=.95\linewidth}
      \caption{Same as~\Cref{fig:chess_alphas}}%
      \label{fig:accidents_alphas}
    \end{subfigure}%
    \begin{subfigure}[t]{0.25\textwidth}
      \scriptsize
      \includegraphics[width=\textwidth]{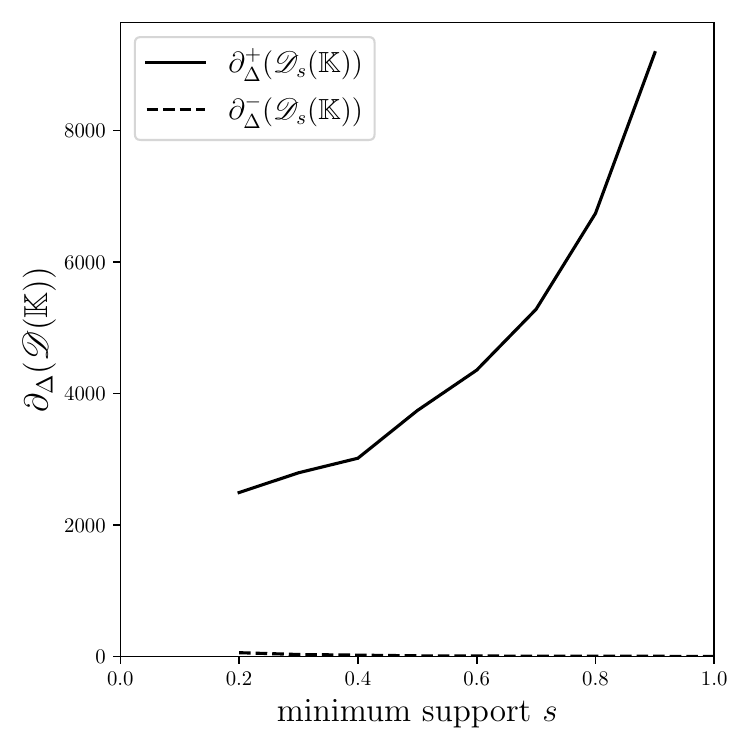}
      \captionsetup{width=.95\linewidth}
      \caption{Same as~\Cref{fig:chess_gid}}%
      \label{fig:accidents_gid}
    \end{subfigure}
    \caption{Results of the computation for the \emph{accidents} data set.}%
    \label{fig:accidents_result}
  \end{figure}
  \columnbreak%
  \begin{figure}[H]
    \begin{subfigure}[t]{0.25\textwidth}
      \scriptsize
      \includegraphics[width=\textwidth]{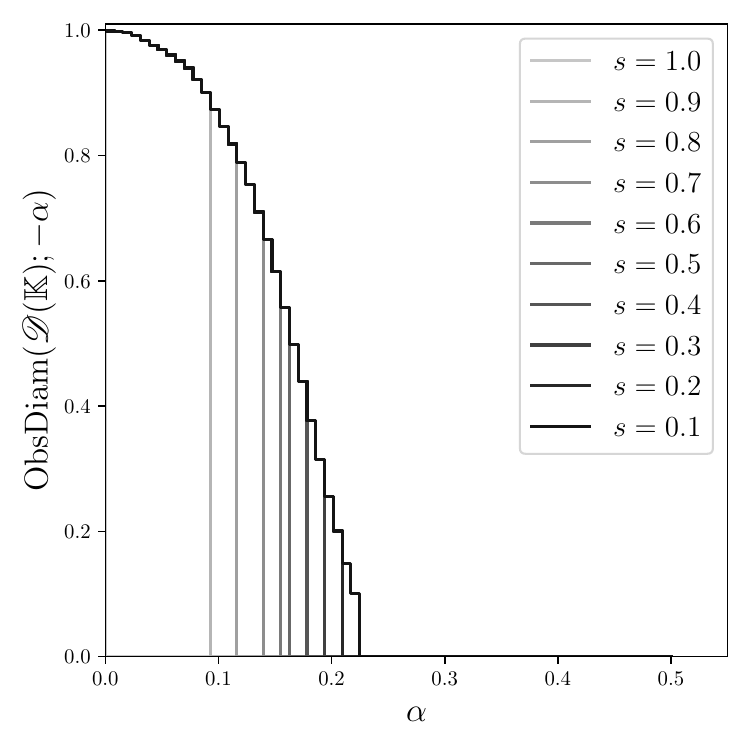}
      \captionsetup{width=.95\linewidth}
      \caption{Same as~\Cref{fig:chess_alphas}}%
      \label{fig:connect_alphas}
    \end{subfigure}%
    \begin{subfigure}[t]{0.25\textwidth}
      \scriptsize
      \includegraphics[width=\textwidth]{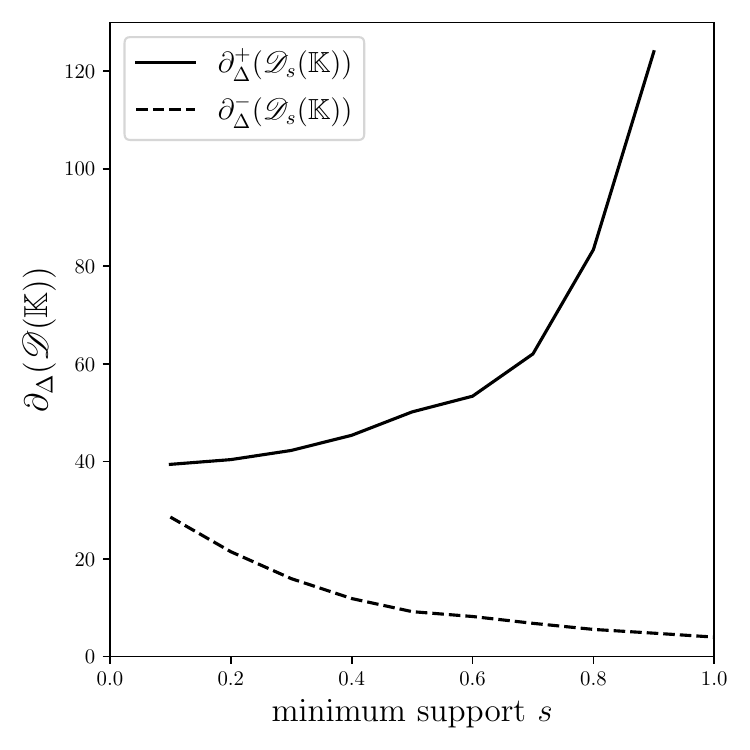}
      \captionsetup{width=.95\linewidth}
      \caption{Same as~\Cref{fig:chess_gid}}%
      \label{fig:connect_gid}
    \end{subfigure}
    \caption{Results of the computation for the \emph{connect} data set.}%
    \label{fig:connect_result}
  \end{figure}
\end{multicols}
\begin{multicols}{2}
  \begin{figure}[H]
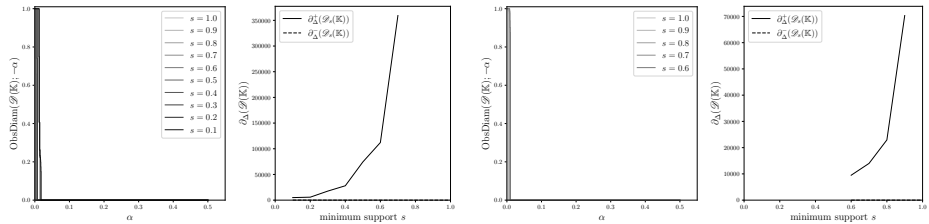

    \begin{subfigure}[t]{0.25\textwidth}
      \scriptsize
      \includegraphics[width=\textwidth]{figures/generated/alphas/pumsb\_star.pdf}
      \captionsetup{width=.95\linewidth}
      \caption{Same as~\Cref{fig:chess_alphas}}%
      \label{fig:pumsb_star_alphas}
    \end{subfigure}%
    \begin{subfigure}[t]{0.25\textwidth}
      \scriptsize
      \includegraphics[width=\textwidth]{figures/generated/gid/pumsb\_star.pdf}
      \captionsetup{width=.95\linewidth}
      \caption{Same as~\Cref{fig:chess_gid}}%
      \label{fig:pumsb_star_gid}
    \end{subfigure}
    \caption{Results of the computation for the \emph{pumsb\_star} data set.}%
    \label{fig:pumsb_star_result}
  \end{figure}
  \columnbreak%
  \begin{figure}[H]
    \begin{subfigure}[t]{0.25\textwidth}
      \scriptsize
      \includegraphics[width=\textwidth]{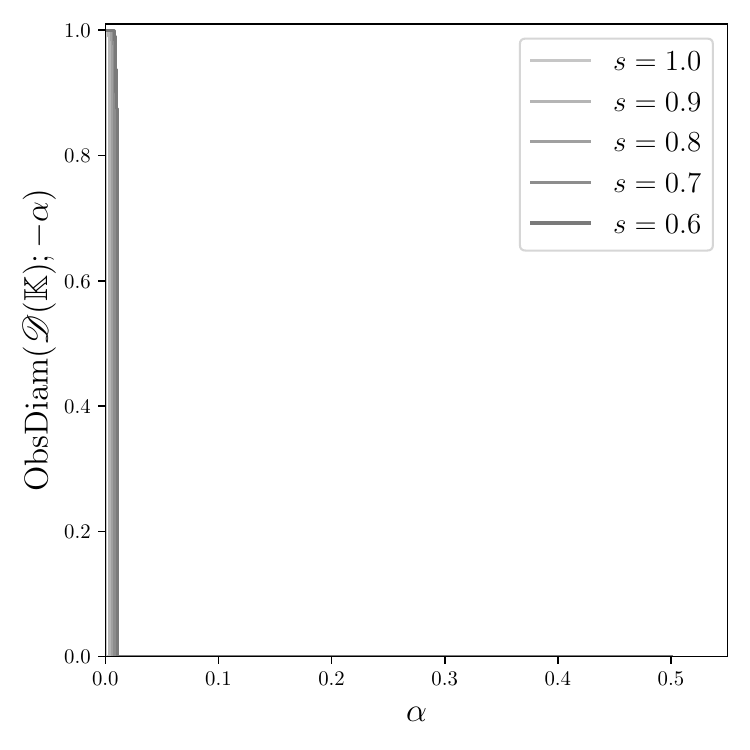}
      \captionsetup{width=.95\linewidth}
      \caption{Same as~\Cref{fig:chess_alphas}}%
      \label{fig:pumsb_alphas}
    \end{subfigure}%
    \begin{subfigure}[t]{0.25\textwidth}
      \scriptsize
      \includegraphics[width=\textwidth]{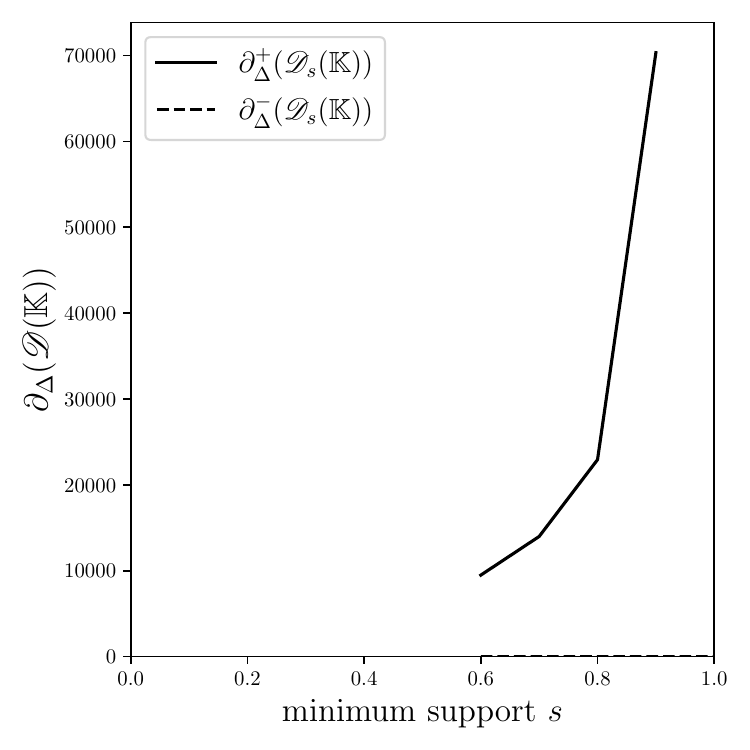}
      \captionsetup{width=.95\linewidth}
      \caption{Same as~\Cref{fig:chess_gid}}%
      \label{fig:pumsb_gid}
    \end{subfigure}
    \caption{Results of the computation for the \emph{pumsb} data set.}%
    \label{fig:pumsb_result}
  \end{figure}
\end{multicols}
\begin{figure}[H]
  \begin{subfigure}[t]{0.25\textwidth}
    \scriptsize
    \includegraphics[width=\textwidth]{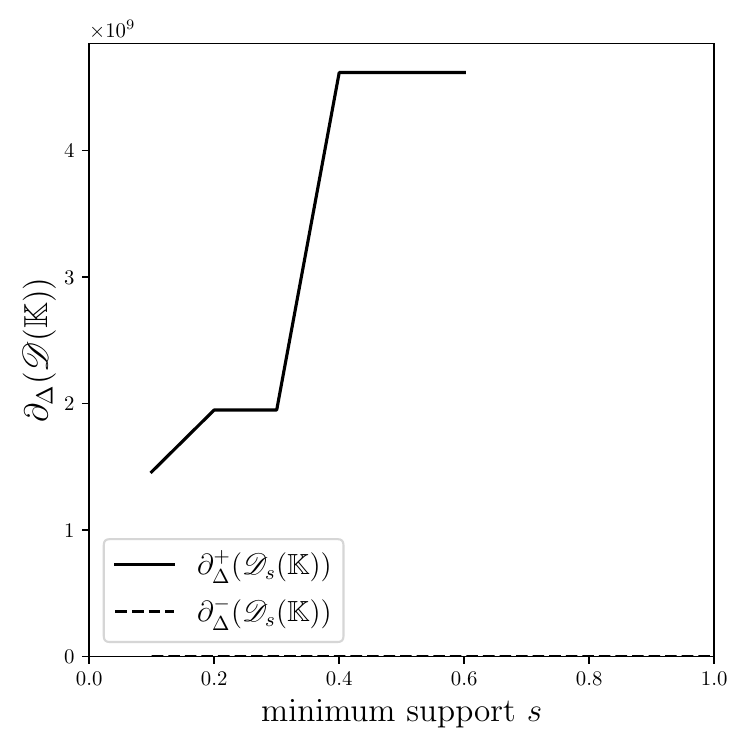}
    \captionsetup{width=.95\linewidth}
    \caption{Same as~\Cref{fig:chess_gid}}%
    \label{fig:kosarak_gid}
  \end{subfigure}%
  \begin{subfigure}[t]{0.25\textwidth}
    \scriptsize
    \includegraphics[width=\textwidth]{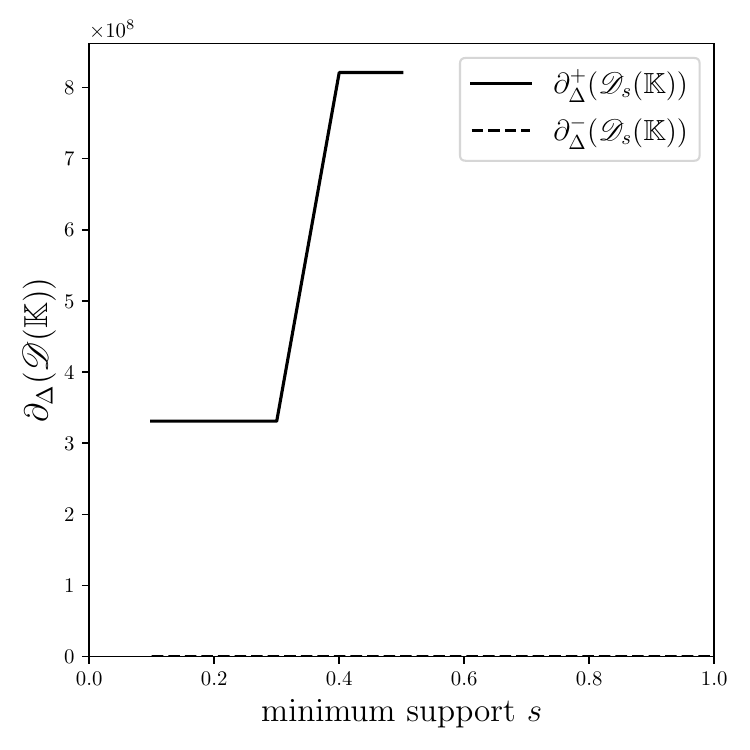}
    \captionsetup{width=.95\linewidth}
    \caption{Same as~\Cref{fig:chess_gid}}%
    \label{fig:pumsb_gid}
  \end{subfigure}
  \begin{subfigure}[t]{0.25\textwidth}
    \scriptsize
    \includegraphics[width=\textwidth]{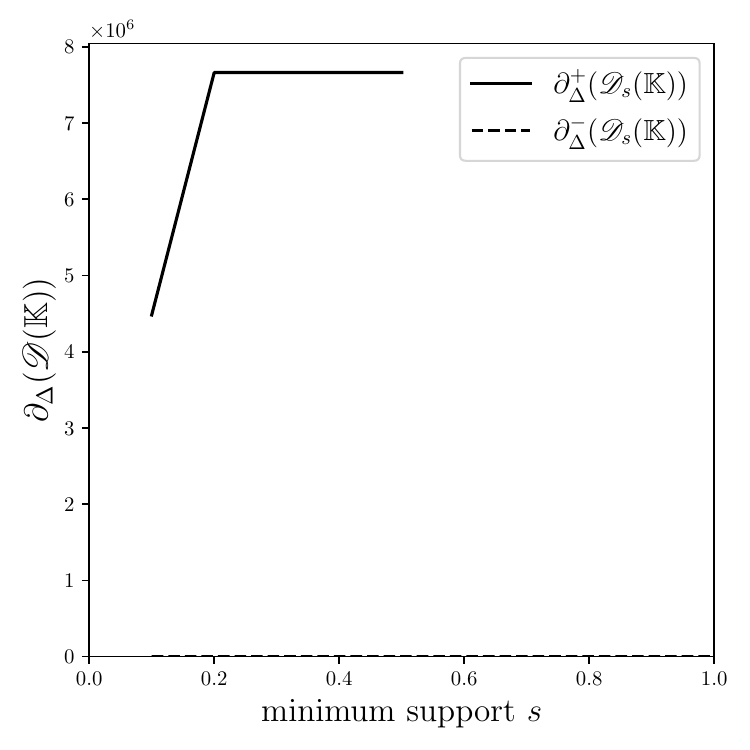}
    \captionsetup{width=.95\linewidth}
    \caption{Same as~\Cref{fig:chess_gid}}%
    \label{fig:retail_gid}
  \end{subfigure}%
  \begin{subfigure}[t]{0.25\textwidth}
    \scriptsize
    \includegraphics[width=\textwidth]{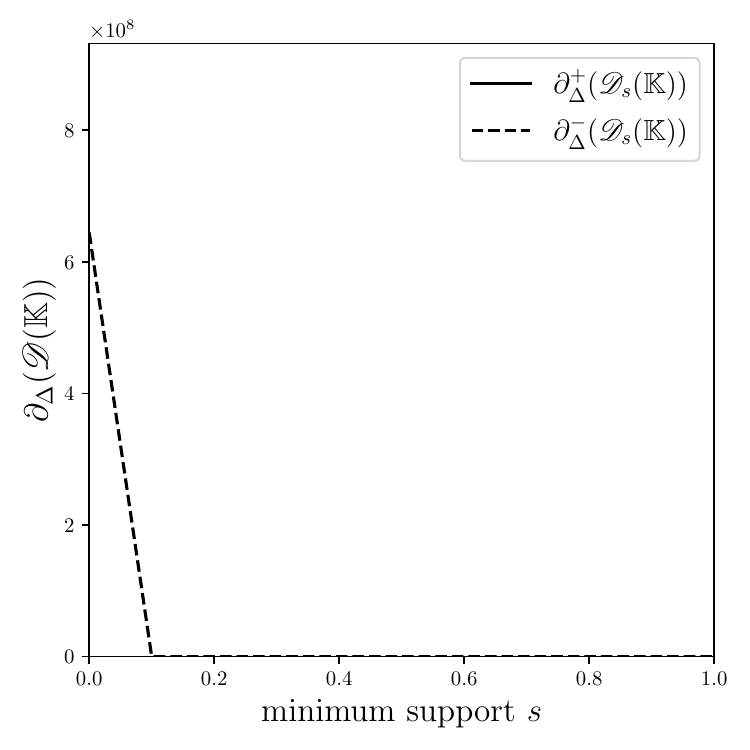}
    \captionsetup{width=.95\linewidth}
    \caption{Same as~\Cref{fig:chess_gid}}%
    \label{fig:pos_gid}
  \end{subfigure}
  \caption{Remaining results of the computation for the \emph{kosarak, retail,
      pos} and \emph{webview-2} data sets. Without diagram of the ObsDiam
    function, as there is hardly any difference to~\Cref{fig:pumsb_alphas}.}%
  \label{fig:remaining_result}
\end{figure}
\end{document}